\documentclass[twoside]{article}

\usepackage[accepted]{aistats2023}

\usepackage{hyperref}       
\usepackage{url}            
\usepackage{booktabs}       
\usepackage{amsfonts}       
\usepackage{nicefrac}       
\usepackage{microtype}      
\usepackage{xcolor}         
\usepackage[font=footnotesize]{subcaption}
\hypersetup{
  colorlinks   = true, 
  urlcolor     = blue, 
  linkcolor    = blue, 
  citecolor    = blue 
}

\usepackage{amssymb,amsmath,amsthm,mathtools,dsfont}
\usepackage{algorithm,algpseudocode}
\usepackage{caption}

\newtheorem{theorem}{Theorem}[section]
\newtheorem{assumption}[theorem]{Assumption}
\newtheorem{lemma}[theorem]{Lemma}
\newtheorem{definition}[theorem]{Definition}

\newcommand{\norm}[1]{\left\lVert#1\right\rVert}

\usepackage[round]{natbib}

\begin{document}
%

%
\runningauthor{W. Xu, J. Ma, K. Xu, H. Bastani, and O. Bastani}

\twocolumn[

\aistatstitle{Uniformly Conservative Exploration in Reinforcement Learning}

\aistatsauthor{%
Wanqiao Xu\\
Stanford University\\
\texttt{wanqiaoxu@stanford.edu } \\
\And
Jason Yecheng Ma \\
University of Pennsylvania \\
\texttt{jasonyma@seas.upenn.edu} \\
\And
Kan Xu \\
University of Pennsylvania \\
\texttt{kanxu@sas.upenn.edu} \\
\AND
Hamsa Bastani \\
University of Pennsylvania \\
\texttt{hamsab@wharton.upenn.edu} \\
\And
Osbert Bastani \\
University of Pennsylvania \\
\texttt{obastani@seas.upenn.edu} \\
}

\aistatsaddress{} ]

\begin{abstract}
A key challenge to deploying reinforcement learning in practice is avoiding excessive (harmful) exploration in individual episodes. We propose a natural constraint on exploration---\textit{uniformly} outperforming a conservative policy (adaptively estimated from all data observed thus far), up to a per-episode exploration budget.
We design a novel algorithm that uses a UCB reinforcement learning policy for exploration, but overrides it as needed to satisfy our exploration constraint with high probability. Importantly, to ensure unbiased exploration across the state space, our algorithm adaptively determines when to explore. We prove that our approach remains conservative while minimizing regret in the tabular setting.
We experimentally validate our results on a sepsis treatment task and an HIV treatment task, demonstrating that our algorithm can learn while ensuring good performance compared to the baseline policy for every patient; the latter task also demonstrates that our approach extends to continuous state spaces via deep reinforcement learning.
\end{abstract}

\section{INTRODUCTION}

Reinforcement learning is a promising approach to learn policies for sequential decision-making to enable data-driven decision-making. For instance, it can be used to help manage health conditions such as sepsis \citep{komorowski2018artificial} and chronic illnesses \citep{zhou2018personalizing}, which require the clinician to make sequences of decisions regarding treatment. Other applications include adaptively sequencing educational material for students \citep{mandel2014offline} or learning inventory control policies with uncertain demand \citep{giannoccaro2002inventory, keller2006automatic}. 

The core challenge in reinforcement learning is how to balance the exploration-exploitation tradeoff---i.e., how to balance taking exploratory actions (to estimate the transitions and rewards of the underlying system) and exploiting the knowledge acquired thus far (to make good decisions). However, in high-stakes settings, exploration can be costly or even unethical---for instance, taking exploratory actions on patients or students can lead to adverse outcomes that could have been avoided.

One solution is \emph{conservative} exploration~\citep{garcelon2020conservative}, where the agent is required to avoid underperforming a baseline policy (a handcrafted heuristic or a policy trained on offline data) by more than some small \emph{exploration budget}. This strategy ensures that the agent does not concentrate exploration (and accrue a large amount of regret) early on; instead, it is forced to balance exploration across time.

However, there are two key shortcomings of conservative exploration. First, it only requires that the learning algorithm outperforms the baseline on \emph{average} across all episodes so far. Thus, the agent could still concentrate exploration on a single episode at a time---indeed, existing algorithms for conservative exploration use exactly such a strategy. Concentrating exploration in a single episode remains problematic in many settings; for instance, in healthcare settings, episodes may correspond to individual patients, and in education settings, they may correspond to individual students.

Second, this strategy only considers a single, fixed baseline policy. However, in practice, the initial baseline policy may not be very good---e.g., a handcrafted heuristic policy may perform significantly worse than the optimal policy. Ideally, the baseline would be updated over time to account for all observations so far. For instance, if the algorithm has discovered that a treatment achieves good outcomes for the current patient, then it is obligated to use either that treatment or an alternative that is only slightly worse. Taken together, we are interested in the following constraint on exploration:
\begin{quote}
\emph{With high probability, the algorithm should never take actions significantly worse than the ones known to be good based on the knowledge accumulated so far.}
\end{quote}
By ``knowledge accumulated so far'', we mean all observations that have been gathered so far. Next, by ``actions known to be good'', we mean high-value actions according to \emph{offline} (or \emph{batch}) reinforcement learning algorithms~\citep{ernst2005tree,levine2020offline}, which are designed to provide conservative estimates of the value function based on historical data~\citep{yu2020mopo,kumar2020conservative}.
Then, our constraint is that, with high probability, the algorithm never takes an action that is significantly worse than using the current baseline policy. We refer to this constraint as \emph{uniformly conservative exploration}. 

Uniformly conservative exploration is significantly harder to satisfy compared to the existing notion of conservative exploration while achieving sublinear regret. Intuitively, we can achieve conservative exploration by simply using the baseline policy for a certain number of episodes; since the exploration budget is a fraction of the accumulated regret, we will have earned enough slack to use an existing algorithm like UCBVI~\citep{azar2017minimax} for an entire episode. By using UCBVI continuously for entire episodes in this manner, we can sufficiently explore the entire state space, ensuring sublinear regret.

However, this strategy no longer works when the exploration constraint must hold for \textit{each} episode, since UCBVI may not be able to explore for a full episode. In particular, consider a strategy where we use UCBVI at the start of each episode until we exhaust our exploration budget, and then switch to the baseline policy. Then, UCBVI will only get to explore near the beginning of each episode, failing to learn about states that can only be visited later in the episode, yielding linear regret.

To remedy this issue, we propose an algorithm that adaptively determines when to explore, with the goal of ``stitching'' together exploration across multiple episodes; these form a single \emph{meta-episode}, which is equivalent to the information gained from using UCBVI for an entire episode. To do so, our algorithm records the state where it switches from the UCBVI policy to the baseline policy, and then only restarts using the UCBVI policy once it encounters the same state in a future episode. We prove that our algorithm explores uniformly conservatively, and obtains regret guarantees similar to those of UCBVI in the number of episodes---i.e., the cost of our constraint is only a constant factor.

Finally, we test the performance of our algorithm on two real-world tasks: learning treatments for sepsis and human immunodeficiency virus (HIV). The latter task has a continuous state space; to this end, we leverage a natural extension of our approach to a deep reinforcement learning algorithm. Our results show that our algorithm can learn as efficiently as existing reinforcement learning algorithms while significantly reducing violations of our uniformly conservative exploration constraint.


Our main contributions in this paper include:
\begin{itemize}
    \item We propose a new notion of \textit{uniformly conservative exploration} (equation~\eqref{eqn:safety}) for reinforcement learning;
    \item We design a novel \textit{meta-episodic} online reinforcement learning algorithm that satisfies our exploration constraint and achieves sublinear regret;
    \item We empirically demonstrate the conservativeness and efficiency of our algorithm on real-world cases in learning treatments for sepsis and HIV.
\end{itemize}

\section{PROBLEM FORMULATION}

\textbf{Preliminaries.}
Consider a Markov decision process (MDP) $M=\langle \mathcal{S}, \mathcal{A}, P, R\rangle$, with finite states $\mathcal{S}$, finite actions $\mathcal{A}$, transition probability $P(s' \mid s,a)$, rewards $R(s,a)\in[0, 1]$\footnote{This is only for simplicity. One can always rescale our result according to the scale of the rewards.}, and time horizon $H\in\mathbb{N}$\footnote{The horizon $H$ is the length of episode.}, where $s', s\in \mathcal{S}$ and $a\in\mathcal{A}$. Thus, we have $S=|\mathcal{S}|$ states and $A=|\mathcal{A}|$ actions. Our analysis is based on $N$ episodes. 
We consider policies $a=\pi_t(s,z)$ with internal state $z\in Z$, along with internal state transitions $z'=\sigma_t(s,z,a)$ for each step $t\in[H]$. Our uniformly conservative exploration property (described in \eqref{eqn:safety}) is a constraint on the reward accrued by our policy across multiple steps in the MDP; thus, our policy uses an internal state to track this information and ensure that our policy satisfies this property.

We define a \emph{rollout} as a random sequence of length $H$, i.e., $\alpha=((s_1,a_1,r_1,s_2),\cdots,(s_H,a_H,r_H))$, where $a_t=\pi_t(s_t,z_t)$, $r_t=R(s_t,a_t)$, $s_{t+1}\sim P(\cdot\mid s_t,a_t)$, and $z_{t+1}=\sigma_t(s_t,z_t,t)$. We assume $s_1$ is deterministic and $z_1$ is given. We denote the distribution over rollouts by $\alpha\sim D_{\pi,\sigma}(\cdot)$, and the rollout of episode $k\in[N]$ by $\alpha_k$. Define the $Q$ function as
\begin{multline*}
Q^{(\pi,\sigma)}_t(s,z,a)=R(s,a) + \sum_{s'\in S}P(s'\mid s,a) \\
\cdot V^{(\pi,\sigma)}_t(s',\sigma_t(s,z,a))
\end{multline*}
with $Q^{(\pi, \sigma)}_H(s,z,a)=0$, and the value function as
\begin{align*}
V^{(\pi,\sigma)}_t(s,z)&=Q^{(\pi,\sigma)}_t(s,z,\pi_t(s, z)),
\end{align*}
with $V^{(\pi, \sigma)}_H(s,z)=0$. 

\textbf{Regret.}
We let $\pi^*_t(s)$ denote the (deterministic) optimal policy, and $Q^*_t(s,a)$ and $V^*_t(s)$ the optimal $Q$- and value functions respectively. $P$ and $R$ are initially unknown. At episode $k\in[N]$, we choose a policy $(\pi^k,\sigma^k)$ along with an initial internal state $z_{k,1}$ based on the observations so far, and observe a new rollout $\alpha_k\sim D_{\pi^k,\sigma^k}(\cdot)$. Our goal is to choose $(\pi^k,\sigma^k)$ and $z_{k,1}$ to minimize the \emph{cumulative regret}
\begin{align*}
\rho=\mathbb{E}\left[\sum_{k=1}^NV_1^*(s_1)-V_1^{(\pi^k,\sigma^k)}(s_1,z_{k,1})\right],
\end{align*}
where the expectation is taken over the randomness of the rollouts $\{\alpha_1, \cdots, \alpha_N\}$.

\textbf{Uniformly conservative exploration.}
Intuitively, our exploration constraint says we do not take actions in an episode that achieve significantly worse rewards than a baseline policy $\bar{\pi}^k$ trained on all observations so far (for simplicity, we assume $\bar{\pi}^k$ is only updated at the end of an episode). Then, it ensures that we do not take harmful action sequences that would have been avoided by $\bar{\pi}^k$.

The strength of the exploration constraint depends on $\bar{\pi}^k$; thus, these bounds should be as tight as possible to avoid harm. We build on a UCB strategy called UCBVI~\citep{azar2017minimax}, a state-of-the-art algorithm that achieves minimax regret guarantees. This algorithm constructs policies based on values that are optimistic compared to the true values; its minimax guarantees stem from the fact that its confidence intervals around its value estimates are very tight. We modify UCBVI to instead construct policies based on conservative values, thereby resulting in a variant of conservative $Q$-learning, an offline reinforcement learning algorithm~\citep{kumar2020conservative}. We describe our approach in detail in Section~\ref{sec:alg}.

Now, given $\eta,\delta\in\mathbb{R}_{>0}$, our exploration constraint says that with probability at least $1-\delta$ (over the randomness of $\{\alpha_k\}_{k\in[N]}$), for every $k\in[N]$ and $t\in[H]$, we have
\begin{align}
\label{eqn:safety}
z_t^*&\coloneqq\sum_{\tau=1}^t\max\left\{V^{(\bar{\pi}^k)}_\tau(s_{k,\tau})-Q^{(\bar{\pi}^k)}_\tau(s_{k,\tau},a_{k,\tau}),0\right\}\le\eta.
\end{align}
We call $z_t^*$ the \emph{reward deficit}, since it is the deficit in reward compared to $\bar{\pi}^k$, and $\eta$ the \emph{exploration budget}, since it bounds how much exploration we can do.
\begin{definition}\label{defn:safety}
An algorithm $\pi$ is uniformly conservative if equation~\eqref{eqn:safety} is satisfied for any $k\in[N]$ and $t\in[H]$ with at least a probability of $1-\delta$.
\end{definition}

To understand (\ref{eqn:safety}), consider the alternative
\begin{align}
\label{eqn:safetynaive}
&V^{(\bar\pi^k)}_t(s_{k,1})-V^{(\pi^k,\sigma^k)}(s_{k,1},z_{k,1})\nonumber\\
&=\mathbb{E}\left[\sum_{t=1}^HV^{(\bar\pi^k)}_t(s_{k,t})- Q^{(\bar\pi^k)}_t(s_{k,t},a_{k,t})\right]\le\eta,
\end{align}
where the equality follows by a telescoping sum argument (see, e.g., Lemma 2.1 in~\cite{bastani2018verifiable}). Intuitively, (\ref{eqn:safetynaive}) says that our cumulative expected reward must be within $\eta$ of that of $\bar{\pi}^k$ across the entire episode. In contrast, (\ref{eqn:safety}) is significantly stronger, since the maximum ensures that we cannot compensate for performing worse than $\bar{\pi}^k$ in one part of an episode by performing better later.

Note that our algorithm can always use $\bar{\pi}^k$, which satisfies (\ref{eqn:safety}); the challenge is how to take exploratory actions in a way that minimizes regret while exploring uniformly conservatively.

\textbf{Assumptions.}
Ensuring uniformly conservative exploration and sublinear regret is impossible without assumptions on our MDP. Otherwise, any exploration by an agent could lead to a violation. Our first assumption says that the MDP is ergodic (e.g., it is also required for conservative exploration under an infinite horizon~\citep{garcelon2020conservative}). Let $\Pi$ be the set of all deterministic policies.
\begin{assumption}
\label{assump:ergodic}
\rm
 Let $T^{\pi}(s’,s)$ be the minimum time it takes to transition from state $s’$ to state $s$ following policy $\pi$. Then, $\Upsilon \coloneqq \max_{s'\neq s}\max_{\pi\in\Pi}\mathbb{E}[T^{\pi}(s',s)] \le H/2$.
\end{assumption}
Here, $\Upsilon$ is the worst-case diameter of the MDP---i.e., the worst-case time it takes for any policy $\pi$ to reach any state $s$ from any state $s'$. This assumption says that every state is visited by any policy $\pi$; for instance, if there is a state not visited by one of our baseline policies $\bar{\pi}^k$, then we would not be able to explore that state, potentially leading to linear regret. Our second assumption says that any \emph{single} step of exploration in the MDP does not violate our exploration constraint:
\begin{assumption}
\label{assump:qvalue}
\rm
For any $\pi\in\Pi$, $s\in S$ and $a\in A$, we have $V_t^{(\pi)}(s)-Q_t^{(\pi)}(s,a)\le\eta/2$.
\end{assumption}
That is, using an arbitrary action $a$ in state $s$ and then switching to $\pi$ (i.e., $Q_t^{(\pi)}(s,a)$), is not much worse than using $\pi$ (i.e., $V_t^{(\pi)}(s)$). Note that we must at least assume $V_t^{(\bar\pi^k)}(s)-Q_t^{(\bar\pi^k)}(s,a)\le\eta$; otherwise, any exploratory action could potentially violate the constraint.
The stricter $\eta/2$ tolerance enables us to continue to take exploratory steps if we have only accrued error $\le\eta/2$ so far: if the tolerance were $\eta$, then if we take a single step such that $V_t^{(\bar\pi^k)}(s)-Q_t^{(\bar\pi^k)}(s,a)>0$, then at each subsequent step $t'>t$, we cannot take an exploratory action, since we run the risk that $(V_t^{(\bar\pi^k)}(s)-Q_t^{(\bar\pi^k)}(s,a))+(V_{t'}^{(\bar\pi^k)}(s)-Q_{t'}^{(\bar\pi^k)}(s,a))>\eta$, which would violate the constraint.

\section{ALGORITHM}
\label{sec:alg}

The key challenge is how to take exploratory actions to minimize regret while ensuring that our exploration constraint holds. We build on upper confidence bound value iteration (UCBVI)~\citep{azar2017minimax}, which obtains near-optimal regret guarantees for finite-horizon MDPs. Like other UCB algorithms, it relies on \textit{optimism}---i.e., it takes actions that optimize the cumulative reward under optimistic assumptions about its estimates of the MDP. A natural strategy is to use the internal state to keep track of the reward deficit accrued so far; then, we can use the UCBVI policy from the beginning of each episode until we exhaust our exploration budget, after which we switch to the baseline policy.

The challenge is that the UCBVI regret guarantees depend crucially on using the UCBVI policy for the entire horizon, or at least for extended periods of time. The reason is that selectively using UCBVI at the beginning of each episode biases the portions of the state space where UCBVI is used; for instance, if there are some states that are only reached late in the episode, then we may never use UCBVI in these states, causing us to underexplore and accrue high regret.

To avoid this issue, our algorithm uses the UCBVI policy in portions of each episode in a sequence of episodes, such that we can ``stitch'' these portions together to form a single \emph{meta-episode} that is mathematically equivalent to using the UCBVI policy for an entire episode. The cost is that we may require multiple episodes to obtain a single UCBVI episode, which would slow down exploration and increase regret. However, we can show that the number of episodes in a meta-episode is not too large with high probability, so the strategy actually achieves similar regret as UCBVI.

\textbf{Overall algorithm.}
Our algorithm is summarized in Algorithm~\ref{alg:sucrl2}; $m$ indexes a single meta-episode, and $n$ indexes an episode of $m$. To be precise, we use \emph{meta-episode} to refer to an iteration $m$ of the outer loop of Algorithm~\ref{alg:sucrl2}, and \emph{episode} to refer to an iteration $(m,n)$ of the inner loop; we alternatively index episodes by $k$ when referring to the sequence of all episodes. Then, we use \emph{rollout} to refer to the sequence $\alpha_{m,n}$ of observations $(s,a,r,s')$ during an episode, and a \emph{meta-rollout} to refer to the rollout $\hat\alpha_m$ consisting of a subset of the observations in $\{\alpha_{m,1},\cdots,\alpha_{m,N_m}\}$, where $N_m$ is the total number of episodes in meta-episode $m$. In particular, $\hat\alpha_m$ consists of observations $(s,a,r,s')$ where the UCB policy $\hat\pi$ was used; our algorithm uses $\hat\pi$ in a way that ensures that $\hat\alpha_m$ is equivalent to a single rollout sampled from the MDP while exclusively using $\hat\pi$.

At a high level, at the beginning of each episode $k$, our algorithm constructs the baseline policy $\bar{\pi}$ using the current rollouts $U=\{\alpha_1,\cdots,\alpha_{k-1}\}$. Furthermore, at the beginning of each meta-episode $m$, our algorithm constructs the UCBVI policy $\hat{\pi}$ using the current meta-rollouts $\hat{U}=\{\hat\alpha_1,\cdots,\hat\alpha_{m-1}\}$. Then, it obtains a sequence of rollouts using $\tilde{\pi}$, which combines the current $\bar{\pi}$ and $\hat{\pi}$ to explore uniformly conservatively. It does so in a way that it can ``stitch'' together portions of the rollouts using $\hat{\pi}$ into a single rollout $\hat{\alpha}_m$ whose distribution equals the distribution over rollouts induced by using $\hat{\pi}$. In other words, $\hat{\alpha}_m$ is equivalent to using $\hat{\pi}$ for a single episode. Thus, each meta-rollout of our algorithm corresponds to a single UCBVI episode. As long as the number of episodes per meta-episode is not too large, we obtain similar regret as UCBVI. We detail our algorithm below.

\begin{algorithm}[h]
\begin{algorithmic}
\Procedure{UnifConservUCBVI}{$M,N,\delta$}
\State Initialize rollout history $U\gets\varnothing$
\State Initialize meta-rollout history $\hat{U}\gets\varnothing$
\For{$m\in\mathbb{N}$}
\State Compute $\hat{\pi}$ using $\hat{U}$
\State Initialize target state $s'\gets s_1$
\For{$n\in\mathbb{N}$}
\State Compute $\bar{\pi}$, $\hat{V}^{(\bar{\pi})}$, and $\bar{Q}^{(\bar{\pi})}$ using $U$
\State Obtain a rollout $\alpha_{m,n}$ using $z_1=(s',0)$, $\sigma$ as in (\ref{eqn:sigma}), and $\tilde{\pi}$ as in (\ref{eqn:pi}), and add it to $U$
\State Update $s'$ to be the next target state, or break if done (and terminate if $|U|\ge N$)
\EndFor
\State Construct $\hat{\alpha}_m$ from $\alpha_{m,1},\cdots,\alpha_{m,N_m}$ and add it to $\hat{U}$
\EndFor
\EndProcedure
\end{algorithmic}
\caption{Uniformly Conservative UCBVI}
\label{alg:sucrl2}
\end{algorithm}

\textbf{Uniformly conservative exploration.}
Our algorithm ensures uniformly conservative exploration by using the policy internal state to keep track of the reward deficit. In particular, suppose we have $\hat{V}_t^{(\bar\pi)}$ satisfying $\hat{V}_t^{(\bar\pi)}(s)\ge V_t^{(\bar\pi)}(s)$ and $\bar{Q}_t^{(\bar{\pi})}$ satisfying $\bar{Q}_t^{(\bar\pi)}(s,a)\le Q_t^{(\bar\pi)}(s,a)$ with high probability; then, we use internal state $z_1=0$ and
\begin{align*}
\sigma_t(s,z,a)&=z+\max\{\hat{V}_t^{(\bar\pi)}(s)-\bar{Q}_t^{(\bar\pi)}(s,a),0\}\\
&=z+\hat{V}_t^{(\bar\pi)}(s)-\bar{Q}_t^{(\bar{\pi})}(s,a),
\end{align*}
where the second equality follows since we always have $\hat{V}_t^{(\bar\pi)}(s)\ge\bar{Q}_t^{(\bar{\pi})}(s,a)$. In particular, $z_t\ge z_t^*$ with high probability. Then, our algorithm switches to $\bar{\pi}$ as soon as $z_t>\eta/2$ (i.e., $z_{t-1}\le\eta/2$)---i.e., it uses the \emph{shield policy}
\begin{align*}
\tilde{\pi}_t(s,z)=\begin{cases}
\hat{\pi}_t(s)&\text{if}~z_t\le\eta/2 \\
\bar{\pi}_t(s)&\text{otherwise},
\end{cases}
\end{align*}
where $\hat{\pi}$ is the current UCBVI policy. Thus, we have
\begin{align*}
z_t^*\le z_t\le z_{t-1}+\eta/2\le\eta,
\end{align*}
where the second inequality follows by Assumption~\ref{assump:qvalue}. Since using $\bar\pi$ does not increase the reward deficit, $z_H^*\le\eta$, so (\ref{eqn:safety}) holds---i.e., $\tilde{\pi}$ ensures the exploration constraint with high probability.

\textbf{Meta-episodes.}
As defined, $\tilde{\pi}$ implements the na\"{i}ve strategy of using $\hat{\pi}$ at the beginning of each episode, and switching to $\bar{\pi}$ if it can no longer satisfy the exploration constraint. However, as discussed above, this strategy may explore the state space in a biased way, accruing linear regret. Instead, we modify $\tilde{\pi}$ to construct a single UCBVI episode (called a \emph{meta-episode}) across multiple actual episodes, which ensures exploration equivalent to UCBVI. We denote such a meta-episode by $m\in[M]$ and an episode in meta-epsiode $m$ by $n\in[N_m]$ (i.e., there are $N_m$ episodes in $m$, so we have $N=\sum_{m=1}^MN_m$ total episodes); we index our episodes by $(m,n)$ instead of $k$.

At a high level, in the first episode of a meta-episode $m$ (i.e., $n=1$), we use $\hat{\pi}$ from the beginning. If $\tilde{\pi}$ uses $\hat{\pi}$ for the entire episode, then this single episode is equivalent to a UCBVI episode, so we are done. Otherwise, we switch to using $\bar{\pi}$ at some step $t$ (i.e., at state $s_{m,1,t}$). Then, in the next episode, we initially use $\bar{\pi}$ until some step $t'$ such that $s_{m,2,t'}=s_{m,1,t}$; at this point, we switch to $\hat{\pi}$ until we have exhausted our exploration budget. If we do not encounter $s_{m,1,t}$, then we try again in the next episode; since the MDP is ergodic, we are guaranteed to find $s_{m,1,t}$ after a few tries with high probability. We continue this process until we have used $\hat{\pi}$ for $H$ steps (i.e., a full UCB episode).
Formally, we augment the internal state of our policy with the target state $s$ from which we want to continue using $\hat{\pi}^m$ (or $s_1$ for the initial episode), so $z=(s',\zeta)\in S\times\mathbb{R}$. In particular, we let
\begin{align}
\label{eqn:z}
z_{m,n,1}&=\begin{cases}
(s_1,0)&\text{if}~n=1 \\
(s_{m,n}',0)&\text{otherwise},
\end{cases}
\end{align}
where $s_{m,n}'$ is the target state for episode $n$---i.e., the state $s_{m,n',t}$ at which we switched to $\bar{\pi}$ for some $n'<n$, such that we did not encounter $s_{m,n',t}$ in episodes $n'<n''<n$. Next, we have 
\begin{align}
\label{eqn:sigma}
&\sigma_t\Big(s,(s',\zeta),a\Big) \nonumber\\
&=\begin{cases}
(s',0)&\text{if}~s'\neq\varnothing~\text{and}~s'\neq s \\
\Big(\varnothing,\zeta+\hat{V}_t^{(\bar\pi)}(s)-\bar{Q}_t^{(\bar{\pi})}(s,a)\Big)&\text{otherwise}.
\end{cases}
\end{align}
That is, the internal state remains $z=(s',0)$ until encountering the target state $s'$; at this point, it becomes $(\varnothing,0)$ and starts accruing reward deficit as before. Finally, we have
\begin{align}
\label{eqn:pi}
\tilde{\pi}_t\Big(s,(s',\zeta)\Big)&=\begin{cases}
\hat{\pi}_t(s)&\text{if}~s'=\varnothing~\text{and}~\zeta\le\eta/2 \\
\bar{\pi}_t(s)&\text{otherwise},
\end{cases}
\end{align}
i.e., we use the UCBVI policy $\hat{\pi}$ if we have reached the target state $s'$ and do not risk exceeding our exploration budget; otherwise, we use the backup policy $\bar{\pi}$.

Finally, a meta-episode terminates once we have used $\hat{\pi}$ at least $H$ times across the rollouts $\alpha_{m,1},\cdots,\alpha_{m,n}$; in this case, we have $n=N_m$ episodes in meta-episode $m$. Then, our algorithm constructs the corresponding \emph{meta-rollout} $\hat{\alpha}_m$ by concatenating the portions of $\alpha_{m,1},\cdots,\alpha_{m,n}$ that use $\hat{\pi}$. In the very last episode $\alpha_{m,n}$, we may continue using $\hat{\pi}$ even after we have obtained the necessary $H$ steps using $\hat{\pi}$; we ignore the extra steps so $\hat{\alpha}_m$ is exactly $H$ steps long.

\textbf{Policy construction.}
Finally, we describe how our algorithm constructs the quantities $\bar{Q}^{(\bar{\pi})}$, $\hat{V}_t^{(\bar{\pi})}(s)$, $\bar{\pi}$, and $\hat{\pi}$. The constructions are based on the UCBVI algorithm; in particular, note that on step $m$, $\hat{U}$ is equivalent to a set of $m-1$ UCBVI rollouts, so we can use it to construct a UCBVI policy $\hat{\pi}$ for the $m$th episode.\footnote{By only using meta-episodes to construct $\hat{\pi}$, the meta-episodes exactly mimic the execution of UCBVI; in practice, we can use the entire dataset $U$ to construct $\hat{\pi}$.}
In particular, we construct $\hat{\pi}$ by estimating the transitions and rewards based on the data collected so far (i.e., the tuples $(s,a,r,s')$ collected on steps using the UCBVI policy, so $a=\hat{\pi}(s)$, $r=R(s,a)$, and $s'\sim P(\cdot\mid s,a)$), to obtain
\begin{align*}
\hat{P}(s'\mid s,a)&=\frac{|\{(s,a,\cdot,s')\in U\}|}{N(s,a)}\\
\hat{R}(s,a)&=\frac{\sum_{(s,a,r,\cdot)\in U}r}{N(s,a)}
\end{align*}
where $N(s,a)=|\{(s,a,\cdot,\cdot)\in U\}|$ is the number of observations of state-action pair $(s,a)$ in the data collected so far. Then, we use value iteration to solve the Bellman equations
\begin{align*}
\hat{Q}_t^*(s,a)&=\hat{R}'(s,a)+\gamma\cdot\sum_{s'\in S}\hat{P}(s'\mid s,a)\cdot\hat{V}_{t+1}^*(s')\\
\hat{V}_t^*(s)&=\max_{a\in A}\hat{Q}_t^*(s,a),
\end{align*}
where $\hat{R}'(s,a)=\hat{R}(s,a)+b(s,a;N(s,a))$, where $b(s,a;N)=4H\sqrt{SL/\max\{1,N\}}$
is a bonus term, and where $L=\log(5SAH\sum_{m=1}^M N_m/\delta)$. Finally, we take $\hat{\pi}_t(s)=\operatorname*{\arg\max}_{a\in A}\hat{Q}^*(s,a)$.

We construct $\bar{Q}^{(\bar\pi)}$ and $\hat{V}^{(\bar{\pi})}$ similarly. For $\bar{Q}^{(\bar\pi)}$, we use the above strategy except we subtract the bonus---i.e., letting $\bar{R}'(s,a)=\hat{R}(s,a)-b(s,a;N(s,a))$, we have
\begin{align*}
\bar{Q}_t^*(s,a)&=\bar{R}'(s,a)+\gamma\cdot\sum_{s'\in S}\hat{P}(s'\mid s,a)\cdot\bar{V}_{t+1}^*(s')\\
\bar{V}_t^*(s)&=\max_{a\in A}\bar{Q}_t^*(s,a).
\end{align*}
Then, we take $\bar\pi_t(s)=\operatorname*{\arg\max}_{a\in A}\bar{Q}_t^*(s,a)$. Finally, for $\hat{V}^{(\bar{\pi})}$, we add the bonus, but use value iteration for policy evaluation instead of policy optimization---i.e.,
\begin{align*}
\hat{Q}_t^{(\bar\pi)}(s,a)&=\hat{R}'(s,a)+\gamma\cdot\sum_{s'\in S}\hat{P}(s'\mid s,a)\cdot\hat{V}_{t+1}^{(\bar\pi)}(s')\\
\hat{V}_t^{(\bar\pi)}(s)&=\hat{Q}_t^{(\bar\pi)}(s,\bar\pi(s)).
\end{align*}

\textbf{Deep reinforcement learning.}
We can straightforwardly adapt our algorithm to MDPs with continuous states using deep reinforcement learning. To this end, we replace the conservative $Q$ function $\bar{Q}$ using the $Q$ function learned via conservative $Q$-learning (CQL)~\citep{kumar2020conservative}; we replace the optimistic value function $\hat{V}$ using a value function learned via deep $Q$-learning with an optimistic bonus. Note that the MDP may in general never return to exactly the same state (since the states are continuous); instead, we check if the current state \textit{approximately} matches (e.g., within a small distance) the target state $s'$. Our HIV experiment successfully implements this approach.


\section{THEORETICAL GUARANTEES}
\label{sec:theory}

All our results are conditioned on a high-probability event $\mathcal{E}$ that (i) our confidence sets around the estimated transitions $\hat{P}$ and rewards $\hat{R}$ hold, and (ii) we find the target state $s'$ in a reasonable number of episodes (see Lemma~\ref{lem:n3}). This event holds with probability at least $1-\delta$; see Appendix~\ref{sec:event}.

First, we prove our algorithm satisfies our exploration constraint.
\begin{theorem}\label{thm:safety}
On event $\mathcal{E}$, Algorithm~\ref{alg:sucrl2} satisfies (\ref{eqn:safety}) for all $k\in[N]$.
\end{theorem}
\begin{proof}
First, we show that $z_t\le\eta$ for all $t\in[H]$. Consider following cases at step $t$: (i) if $\tilde\pi$ uses $\hat\pi$, then $z_t\le\eta/2$, (ii) if $\tilde\pi$ switches to $\bar\pi$ on step $t$, then $z_t\le z_{t-1}+\eta/2\le\eta$, and (iii) otherwise, $z_t=z_{t-1}$ remains the same, so the claim follows by induction. As a consequence, it suffices to show that $z_t\ge z_t^*$ on event $\mathcal{E}$. To this end, the following lemma says that the high probability upper and lower bounds $\hat{V}^{(\bar\pi)}$ and $Q_t^{(\bar\pi)}(s,a)$ used to construct $z_t$ are correct.
\begin{lemma}\label{lem:ucblcb}
On event $\mathcal{E}$, for all $s\in S$, $a\in A$, $k\in[N]$, and $t\in[H]$, we have (i) $\bar{Q}_{k,t}^{(\pi)}(s,a)\le Q_t^{(\pi)}(s,a)$, and (ii) $\hat{V}_{k,t}^{(\pi)}(s)\ge V_t^{(\pi)}(s)$.
\end{lemma}
This result is based on standard arguments; we give a proof in Appendix~\ref{sec:lem:ucblcb:proof}. Now, by Lemma \ref{lem:ucblcb},
\begin{align*}
z_t
\ge \sum_{\tau=1}^t\max\left\{V_\tau^{(\bar\pi^k)}(s_{k,\tau})-Q_\tau^{(\bar\pi^k)}(s_{k,\tau},a_{k,\tau}),0\right\}
&=z_t^*
\end{align*}
on event $\mathcal{E}$, so the claim holds.
\end{proof}

Next, we prove that our algorithm has sublinear regret.
\begin{theorem}
\label{thm:regret}
On event $\mathcal{E}$, the cumulative expected regret 
of Algorithm~\ref{alg:sucrl2} is
\begin{multline*}
\rho \le 20HL\sqrt{SAHN}+250H^2S^2AL^2\\
+\frac{960H^3S}{\eta}\sqrt{2ALHN},
\end{multline*}
where $L = \log(5SAH^2MN)$, and where the expectation is taken over the randomness during all of the rollouts taken. Furthermore, letting $T=H\sum_{m=1}^M N_m=HN$ be the total number of time-steps by the end of meta-episode $M$, the regret satisfies $\rho=\tilde{O}(H^3\sqrt{SAT})$.
\end{theorem}
\begin{proof}
The main idea is to bound the regret by the regret of the meta-rollouts (which correspond to UCBVI rollouts), plus the regret of the shield policy $\tilde\pi$ on the remaining steps---i.e., $\rho=\hat{\rho}+\bar{\rho}$, where
\begin{align*}
\hat\rho&=\mathbb{E}\left[\sum_{m=1}^MV_1^*(s_1)-V_1^{(\hat\pi^m)}(s_1)\right]\\
\bar\rho&=\mathbb{E}\left[\sum_{m=1}^M\sum_{n=1}^{N_m}(V_1^*(s_1)-V_1^{(\bar\pi^{m,n})}(s_1))\mathds{1}((n,t)\not\in\hat\alpha_m)\right],
\end{align*}
where $(n,t)\not\in\hat\alpha_m$ denotes that the $t$th step $(s_{m,n,t},a_{m,n,t})$ of episode $n$ is not included in meta-rollout $\hat\alpha_m$. By equivalence to UCBVI, $\hat\rho$ is bounded by the UCBVI regret:
\begin{lemma}
\label{lem:rhohat}
On event $\mathcal{E}$, we have
\begin{align*}
\hat{\rho}\le20HL\sqrt{SAHN}+250H^2S^2AL^2.
\end{align*}
\end{lemma}
The proof is based on the UCBVI regret analysis; for completeness, we give a proof in Appendix~\ref{sec:lem:rhohat:proof}. Thus, we focus on bounding $\bar\rho$. First, we have the straightforward bound,
\begin{align}
\label{eqn:barrhobound}
\bar\rho\le\mathbb{E}\left[\sum_{m=1}^MH(N_m-1)\right],
\end{align}
which follows since the maximum regret during a single episode is $H$ (since the rewards are bounded by $1$), and since we can also omit the steps for which $(n,t)\in\hat\alpha_m$, of which there are exactly $H$.

As a consequence, the key challenge in bounding $\bar{\rho}$ is proving that the number of episodes $N_m$ in a meta-episode becomes small---in particular, once $N_m=1$, then the entire (single) rollout $\alpha_{m,1}$ is part of the meta-rollout $\hat\alpha_m$, so the second term in the regret is zero.

To prove that $N_m$ becomes small, we note that for any episode, one of the following conditions must hold: (i) the exploration budget is exhausted---i.e., $z_H^*\ge\eta/2$, (ii) the algorithm explores using $\hat\pi$ for at least $H/4$ time steps, or (iii) the episode does not reach the target state $s'$ in the first $3H/4$ time steps; in particular, if (iii) does not hold, then either the episode uses $\hat\pi$ for the final $H/4$ steps of that episode (so (ii) holds) or the exploration budget is exhausted (so (i) holds). We let $N_m^1,N_m^2,N_m^3$ denote the number of episodes that satisfy the three respective cases in meta-episode $m$; note that either $N_m=1$ (i.e., always use the UCBVI policy) or $N_m=N_m^1+N_m^2+N_m^3$.
We bound the three possibilities separately. First, we show that number of episodes $N_m^1$ in case (i) is bounded by the UCBVI regret (i.e., the regret of the meta-episode), which is sublinear.
\begin{lemma}\label{lem:Nm}
On event $\mathcal{E}$, we have
\begin{equation*}
N_m^1\le\frac{2}{\eta}\sum_{t=1}^H\hat{V}_t^{(\hat\pi^m)}(\hat{s}_{m,t})-\bar{Q}_t^{(\bar\pi^{m,n})}(\hat{s}_{m,t},\hat{\pi}(\hat{s}_{m,t})).
\end{equation*}
where $L=\log(5SAH\sum_{m=1}^M N_m/\delta)$, and $N_m(s,a)$ is the total number of observations of the state-action pair $(s,a)$ prior to meta-episode $m$.
\end{lemma}
Intuitively, this lemma follows since if our algorithm exhausts the exploration budget, then it explores sufficiently; thus, the number of times $N_m^1$ that the exploration budget is exhausted cannot be too large. We give a proof in Appendix~\ref{sec:lem:Nm:proof}. The left-hand side of the bound is essentially (but not exactly) the UCBVI regret, and we can bound it using the same strategy. In particular, we have:
\begin{lemma}
\label{lem:alg_regret}
On event $\mathcal{E}$, we have
\begin{multline*}
\sum_{m=1}^M\sum_{t=1}^H\hat{V}_t^{(\hat\pi^m)}(\hat{s}_{m,t})-\bar{Q}_t^{(\bar\pi^{m,n_t})}(\hat{s}_{m,t},\hat{\pi}(\hat{s}_{m,t}))\\
\le12H^2S\sqrt{ALHN}.
\end{multline*}
\end{lemma}
The proof is based on the same strategy as UCBVI, so we defer it to Appendix~\ref{sec:lem:alg_regret:proof}. Note that we have summed over meta-episodes $m\in[M]$; later, we use Lemma~\ref{lem:alg_regret} to directly bound $\sum_{m=1}^MN_m^1$.
Next, to bound $N_m^2$, note that we can use $\hat\pi$ for $H/4$ time steps in at most four episodes, since at the end of the fourth episode we would have a complete UCBVI episode (which has length $H$); thus, $N_m^2\le4$. Next, we use the following result to bound $N_m^3$.
\begin{lemma}
\label{lem:n3}
On event $\mathcal{E}$, for any state $s\in S$, a rollout using $\tilde{\pi}$ will reach state $s$ within $3H/4$ time steps after at most $6\log(1/\delta)$ episodes.
\end{lemma}
This result follows applying Markov's inequality in conjunction with Assumption~\ref{assump:ergodic}, which says the MDP $M$ is ergodic; thus, it visits $s$ with high probability early in the rollout. We give a proof in Appendix~\ref{sec:lem:n3:proof}. Finally, we have the following overall bound:
\begin{lemma}
\label{lem:Nmall}
On event $\mathcal{E}$, we have 
\begin{align*}
N_m\le\max\{30N_m^1\log(1/\delta),1\}.
\end{align*}
\end{lemma}
This result follows by the previous lemmas; we give a proof in Appendix~\ref{sec:lem:Nmall:proof}. In summary, we have
\begin{align}
\label{eqn:barrhobound2}
\bar\rho\le30H\log(1/\delta)\mathbb{E}\left[\sum_{m=1}^MN_m^1\right]\le\frac{960H^3S}{\eta}\sqrt{2ALHN},
\end{align}
where the first inequality follows by combining (\ref{eqn:barrhobound}) with Lemma~\ref{lem:Nmall} (which implies $N_m-1\le30\log(1/\delta)N_m^1$), and the second follows from Lemmas~\ref{lem:Nm} \&~\ref{lem:alg_regret}. Finally, Theorem~\ref{thm:regret} follows by combining (\ref{eqn:barrhobound2}) and Lemma~\ref{lem:rhohat}.
\end{proof}


\section{EXPERIMENTS}

\begin{table}[t]
\captionof{table}{HIV Treatment: maximum discounted return of our approach vs. CQL and $\epsilon$-greedy Q-learning over 1,000 episodes. Numbers are in the millions scale (i.e., $\times 10^6$).}\label{table:cql_comparison}
\centering
\begin{tabular}[b]{ll}
      \toprule
      \textbf{ALGORITHM} & \textbf{MAX DISCOUNTED RETURN} \\
      \midrule
      Ours Budget = 40 & $3.84\pm 0.04$\\
      \hline
      Ours Budget = 60 & $3.72\pm 0.09$\\
      \hline
      Ours Budget = 100 & $3.78\pm 0.03$\\
      \hline
      CQL & $3.75\pm 0.09$\\
      \hline
      $\epsilon$-greedy Q-learning & $3.82\pm 0.04$ \\
      \bottomrule
\end{tabular}%
\end{table}

\begin{figure*}[ht]
\centering
         \centering
\includegraphics[width=0.6\linewidth]{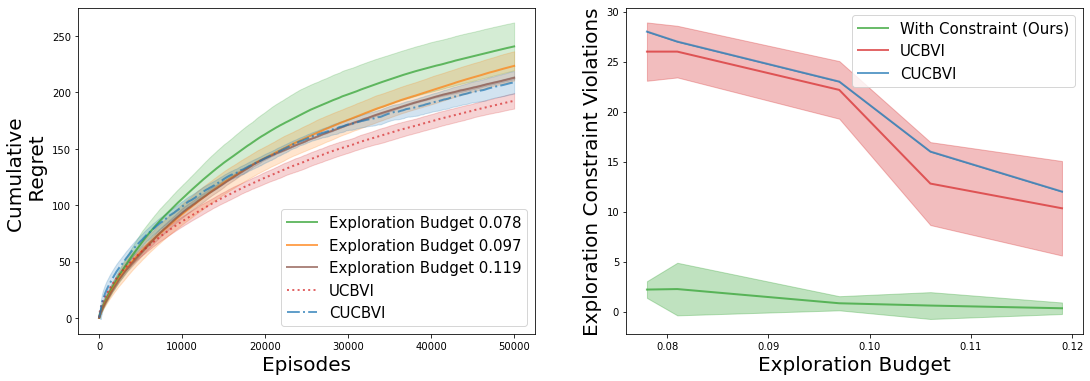}
\captionof{figure}{Sepsis management: regret (left) and exploration constraint violations (right) of our approach with different exploration budgets vs. UCBVI and CUCBVI.}
\label{fig:exp_sepsis}
\vspace{-0.2cm}
\end{figure*}

\begin{figure*}[ht]
\centering
         \centering
         \includegraphics[width=\textwidth]{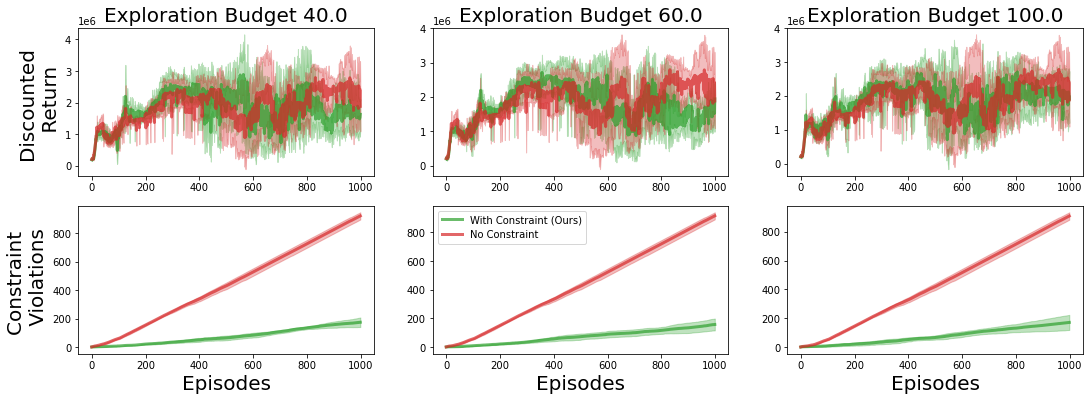}
\captionof{figure}{HIV treatment: discounted return (top) and exploration constraint violations (bottom) of our approach (green) with different exploration budgets vs. deep $Q$-learning (red).}
\label{fig:exp_hiv}
\vspace{-0.5cm}
\end{figure*}

We compare the performance and exploration constraint violations of our algorithm and baseline algorithms on two tasks. We give details on the experimental setup in Appendix~\ref{sec:expappendix}. 

\textbf{Sepsis management.}
To validate our approach in a realistic setting where excessive exploration on individuals is especially harmful, we simulate learning a sepsis treatment policy on the MIMIC-III dataset. Sepsis is the body's acute response to infection that can lead to organ dysfunction, tissue damage, and death. It is the leading cause of hospitalization in the U.S. and the third leading cause of death worldwide. The management of intravenous fluids and vasopressors are crucial in treatment, but current clinical practice is shown to be suboptimal. To develop a more efficient treatment strategy, we can model the problem as an MDP and apply reinforcement learning algorithms. The states are aggregated patient data, and rewards reflect the patient's outcome after medication doses~\citep{komorowski2018artificial}. 

Figure \ref{fig:exp_sepsis} (left) compares the performance of our algorithm with the UCBVI baseline and a conservative benchmark CUCBVI~\citep{garcelon2020conservative}. Our results are averaged over at least 4 trials. The regret of our algorithm converges at similar rates as UCBVI and CUCBVI, and our regret moves closer to UCBVI and CUCBVI as the exploration budget increases. Figure~\ref{fig:exp_sepsis} (right) shows that our algorithm satisfies our exploration constraint most of the time for all exploration budgets, while UCBVI and CUCBVI violate the constraint significantly more even for relatively large exploration budgets. In this setting, each episode represents a patient's treatment cycle, so each constraint violation indicates a patient has received a failed treatment or experienced an adverse outcome. Thus, it is highly undesirable to violate the constraint even for a few episodes.

\textbf{HIV Treatment.} Next, we consider learning an optimal HIV treatment based on the simulation in~\cite{ernst2006hiv}. Acquired immunodeficiency syndrome (AIDS) is a chronic and life-threatening disease caused by HIV. By 2018, there were 36.9 million people living with HIV worldwide and nearly 1 million death caused by AIDS annually~\citep{schwetz2019extended}. To design an optimal drug prescription policy for HIV-infected patients, prior work formulates the problem as a continuous-state MDP that tracks patients' physiological responses to different classes of drugs. We use our algorithm adapted to deep reinforcement learning; our implementation builds on~\cite{killian2017hidden}. 

Figure~\ref{fig:exp_hiv} (up) shows that the performance of our algorithm is comparable to that of Q learning. Figure~\ref{fig:exp_hiv} (down) shows the number of exploration constraint violations of Q learning and our algorithm as a function of the episode. Our results are averaged over at least 9 trials. As can be seen, our algorithm significantly reduces violations compared to vanilla Q learning, while reducing rewards only negligibly. We use conservative Q learning (CQL) as a baseline to our approach, since uniformly conservative exploration is always guaranteed for CQL under our definition. Table~\ref{table:cql_comparison} shows that our algorithm improves the maximum discounted return over CQL via carefully planned exploration instead of always being conservative. These results show that our algorithm successfully extends empirically to continuous-state MDPs. We discuss the extension in more detail in Appendix \ref{sec:expappendix}.




\section{CONCLUSION}

We have proposed a novel reinforcement learning algorithm that ensures close performance compared to our current knowledge uniformly across every step of every episode with high probability. We derive assumptions on the MDP under which both uniformly conservative exploration and sublinear regret can be achieved. Our theoretical results show that the price of uniformly conservative exploration in learning is negligible---i.e., a constant, $T$-independent factor. Our experiments demonstrate that our algorithm can achieve similar performance to state-of-the-art approaches---even in settings with continuous state spaces---while significantly reducing excessive exploration on individual episodes. Our work has ethical considerations insofar as we are proposing a way to reduce the harm of reinforcement learning in practice. Before deploying our approach in any domain, it is critical to ensure that the algorithm does not harm the individuals it impacts, either through excessive exploration (the focus of this work) or other context-specific factors.

\subsubsection*{Acknowledgements}
This work was generously supported by NSF Award CCF-1910769, NSF Award CCF-1917852, ARO Award W911NF-20-1-0080, and a grant from Analytics at Wharton.

{\small
\bibliography{refs}
\bibliographystyle{abbrvnat}
}

\clearpage
\appendix

\onecolumn

\section{Proofs for Section~\ref{sec:theory}}

\subsection{High probability event} 
\label{sec:event}
We first introduce the high probability event $\mathcal{E}$ under which the concentration inequalities described in the policy construction and in UCBVI-CH hold. Let $\mathcal{E}_{1, \delta}$ be the high probability event under which the UCBVI-CH regret analysis holds. This event $\mathcal{E}_{1, \delta}$ is defined in the equation on the bottom of page 16 in the appendices of \cite{azar2017minimax}. The proof that $\mathcal{E}_{1, \delta}$ holds with probability at least $1-\delta$ is proved in the subsequent Lemma 1. We then define
\begin{align*}
c_p(n)&\coloneqq 2\sqrt{\frac{SL}{\max\{1,n\}}},\\
c_r(n)&\coloneqq 2\sqrt{\frac{L}{\max\{1,n\}}}.
\end{align*}
Let $\mathcal{P}$ denote the set of all probability distributions on the states $S$, then construct the confidence sets for every $k=1,\dots,N$ and $(s,a)\in S\times A$
\begin{align*}
B_p^k(s,a)&\coloneqq\Big\{\tilde{R}(\cdot\mid s,a):|\tilde{R}(s,a)-R(s,a)|\le c_r(N_k(s,a))\Big\},\\
B_r^k(s,a)&\coloneqq\Big\{\tilde{P}(\cdot\mid s,a)\in\mathcal{P}:\norm{\tilde{P}(\cdot\mid s,a)-P(\cdot\mid s,a)}_1\le c_r(N_k(s,a))\Big\}.
\end{align*}
Next, we define the random event $\mathcal{E}_{2,\delta}$
\begin{align*}
    \mathcal{E}_{2,\delta}\coloneqq&\bigcap_{(s,a)\in S\times A}\bigcap_{k\in[N]}\left\{\hat{P}_k(\cdot\mid s,a)\in B_p^k(s,a)\right\}\left\{\hat{R}_k(s,a)\in B_r^k(s,a)\right\}.
\end{align*}
where $N_k(s,a)$ is the number of observations of state-action pair $(s,a)$ up to episode $k$. Finally, letting $\mathcal{E}\coloneqq\mathcal{E}_{1,\delta}\cap\mathcal{E}_{2,\delta}$, we conclude that $\mathcal{E}$ holds with probability at least $1-2\delta$. Indeed, 
\begin{align*}
    \Pr(\mathcal{E}_{2,\delta}^c)&\le \sum_{s,a}\sum_{k} \frac{2\delta}{5NSA}\le\delta
\end{align*}
and the claim follows from a union bound.

\subsection{Proof of Lemma \ref{lem:ucblcb}}
\label{sec:lem:ucblcb:proof}

\begin{proof}
First, we prove claim (i). We show by induction that $\bar{Q}_{k,t}^{(\pi)}$ is indeed a lower bound on $Q_t^{(\pi)}$, the real Q functions. Define the sets
\begin{align*}\bar{\Omega}_{k,t}^{(\pi)} &= \{\bar{Q}_{i,j}^{(\pi)} \le Q^{(\pi)}_j, \forall (i,j), i\in [N], j\in [H], i<k\lor(i=k\land j>t)\}
\end{align*}
We want to show that the set of events $\{\bar{\Omega}_{k,t}^{(\pi)}\}_{k\in[K],t\in[H]}$ hold under the event $\mathcal{E}$. 

We proceed by induction. For $t=H$, by definition, $\bar{Q}_{k,H}^{(\pi)}=Q_{H}^{(\pi)}$, so $\bar{Q}_{k,t}^{(\pi)}\le Q_{t}^{(\pi)}$ holds. Now, assuming $\bar{Q}_{k,t+1}^{(\pi)}\le Q_{t+1}^{(\pi)}$ holds, we want to show that $\bar{Q}_{k,t}^{(\pi)}\le Q_{t}^{(\pi)}$ also holds. To this end, note that
\begin{align*}
Q_{t}^{(\pi)}(s,a)-\bar{Q}_{k,t}^{(\pi)}(s,a) &=b_k(s,a)+P(\cdot\mid s,a)V_{t+1}^{(\pi)} - \hat{P}(\cdot\mid s,a)\bar{V}_{k,t+1}^{(\pi)}+R(s,a)-\hat{R}(s,a)\\
&=b_k(s,\pi(s)) + \left(P(\cdot\mid s,a) - \hat{P}(\cdot\mid s,a)\right)V_{t+1}^{(\pi)}(s)\\
&\qquad+\hat{P}(\cdot\mid s,a)\left(V_{t+1}^{(\pi)} - \bar{V}_{k,t+1}^{(\pi)}\right)(s)+R(s,a)-\hat{R}(s,a)\\
&=b_k(s,\pi(s)) + \left(P(\cdot\mid s,a) - \hat{P}(\cdot\mid s,a)\right)V_{t+1}^{(\pi)}(s)\\
&\qquad+\hat{P}\left(\cdot\mid s,a)(Q_{t+1}^{(\pi)} - \bar{Q}_{k,t+1}^{(\pi)}\right)(s,\pi(s))+R(s,a)-\hat{R}(s,a)\\
&\ge b_k(s,\pi(s)) + \left(P(\cdot\mid s,a) - \hat{P}(\cdot\mid s,a)\right)V_{t+1}^{(\pi)}(s)+R(s,a)-\hat{R}(s,a)
\end{align*}
where we use the induction hypothesis in the last inequality. The event $\mathcal{E}$, by H\"older's inequality, implies that
\begin{align*}
&\left|\left(P(\cdot\mid s,a) - \hat{P}(\cdot\mid s,a)\right)V_{t+1}^{(\pi)}(s)+R(s,a)-\hat{R}(s,a)\right|\\ 
&\le\|P(\cdot\mid s,a) - \hat{P}(\cdot\mid s,a)\|_1 \|V_{t+1}^{(\pi)}\|_{\infty}+2\sqrt{L/\max\{1,N_k(s,\pi(s))\}}\\
&\le 2H\sqrt{SL/\max\{1,N_k(s,\pi(s))\}}+2\sqrt{L/\max\{1,N_k(s,\pi(s))\}}\\
&\le b_k(s,\pi(s))
\end{align*}
as claimed. Next, we prove claim (ii), again by backwards induction. Define the sets
\begin{align*}
\hat{\Omega}_{k,t}^{(\pi)} &= \{\hat{V}_{i,j}^{(\pi)} \ge V^{(\pi)}_j, \forall (i,j), i\in [N], j\in [H],i<k\lor(i=k\land j>t)\}.
\end{align*}
We want to show that the set of events $\{\hat{\Omega}_{k,t}^{(\pi)}\}_{k\in[K],t\in[H]}$ hold under the event $\mathcal{E}$. Again we proceed by induction. By definition, $\hat{V}_{k,H}^{(\pi)}=V_{H}^{(\pi)}$. Assuming $\hat{V}_{k,t+1}^{(\pi)}\ge V_{t+1}^{(\pi)}$ holds, we want to show that $\hat{V}_{k,t}^{(\pi)}\ge V_{t}^{(\pi)}$ also holds. To this end, note that
\begin{align*}
\hat{V}_{k,t}^{(\pi)}(s)-V_{t}^{(\pi)}(s)&=b_k(s,\pi(s)) +\hat{P}_k^{(\pi)}\hat{V}_{k,t+1}^{(\pi)}(s)-P^{(\pi)}V_{t+1}^{(\pi)}(s)+\hat{R}(s,\pi(s))-R(s,\pi(s))\\
&=b_k(s,\pi(s)) + (\hat{P}_k^{(\pi)}-P^{(\pi)})V_{t+1}^{(\pi)}(s) + \hat{P}_k^{(\pi)}(\hat{V}_{k,t+1}^{(\pi)}-V_{t+1}^{(\pi)})(s)\\
&\qquad+\hat{R}(s,\pi(s))-R(s,\pi(s))\\
&\ge b_k(s,\pi(s)) + (\hat{P}_k^{(\pi)}-P^{(\pi)})V_{t+1}^{(\pi)}(s)+\hat{R}(s,\pi(s))-R(s,\pi(s))
\end{align*}
where we use the induction hypothesis in the last inequality. The event $\mathcal{E}$, by H\"older's inequality, implies that
\begin{align*}
&|(\hat{P}_k^{(\pi)}-P^{(\pi)})V_{t+1}^{(\pi)}(s)+\hat{R}(s,\pi(s))-R(s,\pi(s))|\\ 
&\le||\hat{P}_k^{(\pi)}-P^{(\pi)}||_1 ||V_{t+1}^{(\pi)}||_{\infty}+2\sqrt{L/\max\{1,N_k(s,\pi(s))\}}\\
&\le 2H\sqrt{SL/\max\{1,N_k(s,\pi(s))\}}+2\sqrt{L/\max\{1,N_k(s,\pi(s))\}}\\
&\le b_k(s,\pi(s))
\end{align*}
as claimed.
\end{proof}

\subsection{Proof of Lemma~\ref{lem:rhohat}}
\label{sec:lem:rhohat:proof}

\begin{proof}
Note that
\begin{align*}
\hat\rho&\coloneqq \mathbb{E}\left[\sum_{m=1}^M V_1^*(s_1)-V_1^{({\hat\pi}^m)}(s_1)\right]\\
&=\sum_{m=1}^M\sum_{(n,t)\in\hat\alpha_{m}}\mathbb{E}\left[V_t^*(s_{m,n,t})-Q_t^*(s_{m,n,t},\hat\pi(s_{m,n,t}))\right]\\
&=\sum_{m=1}^M\sum_{t=1}^H\mathbb{E}\left[V_t^*(\hat{s}_{m,t})-Q_t^*(\hat{s}_{m,t},\hat\pi(\hat{s}_{m,t}))\right]
\end{align*}
where the last equality follows after relabeling the steps in the UCBVI pseudo-episodes. We then apply the same argument in the proof of Theorem 1 in \cite{azar2017minimax}. Note that the pigeon-hole principle only works if we take the total time steps in the theorem to be the total time steps of the whole $M$ meta-episodes. Therefore, the desired bound holds with probability at least $1-\delta$. 
\end{proof}

\subsection{Proof of Lemma~\ref{lem:Nm}}
\label{sec:lem:Nm:proof}

\begin{proof}
First, we sum the condition $z_{m,n,H}^*\ge\eta/2$ over episode $n\in[N_m]$, which gives
\begin{align*}
N_m^1\cdot\frac{\eta}{2}\le\sum_{n=1}^{N_m}z_{m,n,H}^*
=\sum_{n=1}^{N_m} \sum_{t=1}^H \max\Big\{V_t^{(\bar\pi^{m,n})}(s_{m,n,t})-Q_t^{(\bar\pi^{m,n})}(s_{m,n,t},\hat{\pi}(s_{m,n,t})),0\Big\}.
\end{align*}
Now, note that only when $a_{m,n,t}\neq\bar\pi^{m,n}(s_{m,n,t})$, $V_t^{(\bar\pi^{m,n})}(s_{m,n,t})\neq Q_t^{(\bar\pi^{m,n})}(s_{m,n,t},a_{m,n,t})$---i.e., when $(s_{m,n,t},a_{m,n,t})$ is part of the meta-rollout $\hat{\alpha}_m$. Thus, we can restrict the sum to steps in $\hat\alpha_m$:
\begin{align*}
N_m^1\cdot\frac{\eta}{2}
&\le\sum_{t=1}^H \max\Big\{V_t^{(\bar\pi^{m,n_t})}(\hat{s}_{m,t})-Q_t^{(\bar\pi^{m,n_t})}(\hat{s}_{m,t},\hat{\pi}(\hat{s}_{m,t})),0\Big\} \\
&\le\sum_{t=1}^H\hat{V}_t^{(\bar\pi^{m,n_t})}(\hat{s}_{m,t})-\bar{Q}_t^{(\bar\pi^{m,n_t})}(\hat{s}_{m,t},\hat{\pi}(\hat{s}_{m,t})) \\
&\le\sum_{t=1}^H\hat{V}_t^{(\hat\pi^m)}(\hat{s}_{m,t})-\bar{Q}_t^{(\bar\pi^{m,n_t})}(\hat{s}_{m,t},\hat{\pi}(\hat{s}_{m,t})).
\end{align*}
Here, the second line follows since by Lemma~\ref{lem:ucblcb}, $\hat{V}^{(\bar\pi)}$ is an upper bound and $\bar{Q}^{(\bar\pi)}$ is a lower bound on event $\mathcal{E}$, and since $\hat{V}_t^{(\bar\pi)}(s)\ge\bar{V}_t^{(\bar\pi)}(s)\ge\bar{Q}_t^{(\bar\pi)}(s,a)$ for any $a$ since $\bar\pi$ by definition of $\bar\pi$. Finally, the third line follows since $\hat\pi$ is optimistic.
\end{proof}

\subsection{Proof of Lemma \ref{lem:alg_regret}}
\label{sec:lem:alg_regret:proof}

\begin{proof}
By Bellman equations, 
\begin{align*}
\hat{V}_t^{({\hat\pi}^m)}(\hat{s}_{m,t})-\bar{Q}_t^{({\bar\pi}^{m,n_t})}(\hat{s}_{m,t},\hat{\pi}(\hat{s}_{m,t})) & = \bigg(\hat{R}(\hat{s}_{m,t}, \hat\pi(\hat{s}_{m,t}))+\left[P^{(\hat\pi)}\right]^\mathsf{T}\hat{V}^{(\hat\pi)}_{t+1}(\hat{s}_{m,t+1})+b_m(\hat{s}_{m,t},\hat\pi(\hat{s}_{m,t}))\bigg) \\
& \qquad -\bigg(\hat{R}(\hat{s}_{m,t}, \hat\pi(\hat{s}_{m,t}))+\left[P^{(\hat\pi)}\right]^\mathsf{T}\bar{V}^{(\bar\pi)}_{t+1}(\hat{s}_{m,t+1})-b_m(\hat{s}_{m,t},\hat\pi(\hat{s}_{m,t}))\bigg)\\
& \le \max_{q\in B_p^m}(q-p^*)^\mathsf{T}\hat{V}^{(\hat\pi)}_{t+1}(\hat{s}_{m,t+1})
-\min_{q\in B_p^m}(q-p^*)^\mathsf{T}\bar{V}^{(\bar\pi)}_{t+1}(\hat{s}_{m,t+1})\\
& \qquad +2b_m(\hat{s}_{m,t},\hat\pi(\hat{s}_{m,t})) +{p^*}^\mathsf{T}\left(\hat{V}^{(\hat\pi)}_{t+1}-\bar{V}^{(\bar\pi)}_{t+1}\right)(\hat{s}_{m,t+1}).
\end{align*}
We define 
\begin{align*}
(a)_{m,t}&\coloneqq\max_{q\in B_p^m}(q-p^*)^\mathsf{T}\hat{V}^{(\hat\pi)}_{t}(\hat{s}_{m,t})-\min_{q\in B_p^m}(q-p^*)^\mathsf{T}\bar{V}^{(\bar\pi)}_{t}(\hat{s}_{m,t})\\
(b)_{m,t}&\coloneqq2b_m(\hat{s}_{m,t},\hat\pi(\hat{s}_{m,t})).
\end{align*}
Then, for each $t$, note that 
\begin{align*}
\hat{V}_t^{(\hat\pi)}(\hat{s}_{m,t})-\bar{V}_t^{(\bar\pi)}(\hat{s}_{m,t})&=\hat{V}_t^{(\hat\pi)}(\hat{s}_{m,t})-\max_{a\in A}\bar{Q}_t^{(\bar\pi)}(\hat{s}_{m,t},a)
\le\hat{V}_t^{(\hat\pi)}(\hat{s}_{m,t})-\bar{Q}_t^{(\bar\pi)}(\hat{s}_{m,t},\hat\pi(\hat{s}_{m,t})).
\end{align*}
Thus, we have
\begin{align*}
\hat{V}_t^{(\hat\pi)}(\hat{s}_{m,t})-\bar{Q}_t^{(\bar\pi)}(\hat{s}_{m,t},\hat{\pi}(\hat{s}_{m,t}))\le(a)_{m,t+1}+(b)_{m,t}+{p^*}^\mathsf{T}\bigg(\hat{V}_{t+1}^{(\hat\pi)}-\bar{Q}_{t+1}^{(\bar\pi)}(\cdot,\hat\pi)\bigg)(\hat{s}_{m,t+1}).
\end{align*}
Continuing this argument, and noticing that by construction $\hat{V}^{(\hat\pi)}_H(\hat{s}_{m,H})=\bar{Q}^{(\bar\pi)}_H(\hat{s}_{m,H},\hat\pi(\hat{s}_{m,H}))=0$, we have by induction that
\begin{align*}
\hat{V}^{(\hat\pi)}_t(\hat{s}_{m,t})-\bar{Q}^{(\bar\pi)}_t(\hat{s}_{m,t},\hat\pi(\hat{s}_{m,t}))\le\sum_{\ell=0}^{H-t}((a)_{m,t+\ell+1}+(b)_{m,t+\ell})
\end{align*}
Summing over the whole UCBVI episode, under the event $\mathcal{E}$, we have
\begin{align*}
&\sum_{t=1}^H\hat{V}_t^{({\hat\pi}^m)}(\hat{s}_{m,t})-\bar{Q}_t^{({\bar\pi}^{m,n_t})}(\hat{s}_{m,t},\hat{\pi}(\hat{s}_{m,t}))\\
&\le\sum_{t=1}^H\sum_{\ell=0}^{H-t}((a)_{m,t+\ell+1}+(b)_{m,t+\ell})\\
&=\sum_{\ell=1}^H\sum_{t=\ell}^H((a)_{m,t+1}+(b)_{m,t})\\
&\le\sum_{\ell=1}^H\sum_{t=1}^{H}  \bigg(4\sqrt{\frac{SL}{\max\{1,N_m(\hat{s}_{m,t},\hat{a}_{m,t})\}}}+8H\sqrt{\frac{SL}{\max\{1,N_m(\hat{s}_{m,t},\hat{a}_{m,t})\}}}\bigg)\\
&=H\sum_{t=1}^{H}\bigg(4\sqrt{\frac{SL}{\max\{1,N_m(\hat{s}_{m,t},\hat{a}_{m,t})\}}}+8H\sqrt{\frac{SL}{\max\{1,N_m(\hat{s}_{m,t},\hat{a}_{m,t})\}}}\bigg)\\
&=4H\sqrt{SL}\left(1+2H\right)\sum_{t=1}^H\sqrt{\frac{1}{\max\{1,N_m(\hat{s}_{m,t},\hat{a}_{m,t})\}}}
\end{align*}
where $L=\log(5SAH\sum_{m=1}^M N_m/\delta)$. Then, summing over the $M$ meta-episodes, we have
\begin{align*}
&\sum_{m=1}^M\sum_{t=1}^H\hat{V}_t^{({\hat\pi}^m)}(\hat{s}_{m,t})-\bar{Q}_t^{({\bar\pi}^{m,n_t})}(\hat{s}_{m,t},\hat{\pi}(\hat{s}_{m,t}))\\
&\le4H\sqrt{SL}(1+2H)\sum_{m=1}^M\sum_{t=1}^H\sqrt{\frac{1}{\max\{1,N_m(\hat{s}_{m,t},\hat{a}_{m,t})\}}}\\
&\le4H\sqrt{SL}(1+2H)\sum_{s,a}\sum_{n=1}^{N_M(s,a)}\sqrt{\frac{1}{n}}\\
&\le12H^2S\sqrt{ALHN}.
\end{align*}

\end{proof}

\subsection{Proof of Lemma~\ref{lem:n3}}
\label{sec:lem:n3:proof}

\begin{proof}
Under our assumption that the MDP is ergodic, let
\[\Gamma = \max_{s=s'}\max_{\pi}\mathbb{E}[T^{\pi}(s,s')] \le \frac{H}{2}\]
be the worst-case diameter. Then given any initial state $s'$, target state $s$ and shield policy $\tilde{\pi}$, the expected exit time
\[\mathbb{E}[T^{\tilde\pi}(s',s)]\le\Gamma.\]
By Markov's inequality, 
\[\Pr(T^{\tilde\pi}(s',s)\ge\alpha H)\le\frac{1}{2\alpha}.\]
Therefore, with probability at least $1-(\frac{1}{2\alpha})^N$, during $N$ episodes, there exists one where the MDP will reach $s$ from $s'$ using $\tilde\pi$ within $\alpha H$ steps. Letting $N = \log(\frac{1}{\delta})/\log(2\alpha)$ and $\alpha=3/4$ completes the proof.
\end{proof}

\subsection{Proof of Lemma~\ref{lem:Nmall}}
\label{sec:lem:Nmall:proof}

\begin{proof}
If $N_m=1$, then the bound trivially holds. Otherwise, note that we must have $N_m^1\ge1$, since if $N_m\neq1$ then we must have exhausted the exploration budget during the first episode $n=1$. Next, by Lemma~\ref{lem:n3}, we have $N_m^3\le(N_m^1+N_m^2)\max\{6\log(1/\delta)-1,0\}$---i.e., it is bounded by the number of ``successful'' episodes $N_m^1+N_m^2$ times the maximum number of tries $\max\{6\log(1/\delta)-1,0\}$ before finding a successful episode. Together with the fact that $N_m^2\le4$, we have
\begin{align*}
N_m&=N_m^1+N_m^2+N_m^3\\
&\le N_m^1+4+(N_m^1+4)\max\{6\log(1/\delta)-1,0\} \\
&\le5N_m^1+5N_m^1\max\{6\log(1/\delta)-1,0\} \\
&\le30N_m^1\log(1/\delta),
\end{align*}
where on the second line, we have used the fact that we are considering the case $N_m^1\ge1$, which implies $N_m^1+4\le5N_m^1$. The claim follows.
\end{proof}

\section{Experiment Details}
\label{sec:expappendix}

\textbf{Sepsis management.}
We adopt the MDP trained in \cite{komorowski2018artificial} as the underlying MDP we need to learn, set the horizon to $H=20$, and run all tests over $N=50,000$ total episodes. For the sake of completeness, we describe the detailed construction of the MDP as follows. A set of 750 mutually exclusive states encode patients' health states constructed by clustering patients' data. The actions are the dose prescribed of intravenous fluids and vasopressors converted into 25 discrete decisions, with the dose of each treatment discretized into one of five possible dose levels. The transition matrix describes the state transition dynamics, which can be computed via taking sample averages. A positive reward is given at the end of each patient's treatment cycle if the patient survives, and a negative reward is issued if the patient dies. Note that the MDP is hidden to our algorithm, UCBVI and CUCBVI. We also use an offline dataset of 500 randomly generated past episodes to warm-start the algorithms, and then run all algorithms for $N=50,000$ online episodes respectively. We compute the regret and number of exploration constraint violations corresponding to various exploration budgets $\eta$ lying between 0.078 and 0.119. To account for the randomness of each training, we run the experiment for each $\eta$ at least four times and take the average over the regret and constraint violations. The graphs shown in Figure \ref{fig:exp_sepsis} are plotted according to the average values. We also compare our algorithm to CUCBVI introduced in \cite{garcelon2020conservative}

\textbf{HIV treatment.} 
We build on the implementation in \cite{killian2017hidden}. There are 6 state variables represented as a 6-dimensional continuous vector that encodes concentrations of 6 different cells, measured every five days to determine the drug combination for the next five days. There are 4 actions corresponding to 2 drugs being activated or not, measured every five days. In particular, these four on-off combinations of drug administration consist of: RTI (Reverse Transcriptase Inhibitors) and PI (Protease Inhibitors) on, only RTI on, only STI on, RTI and PI off. The horizon is set to 200, which correspond to 1,000 days of monitoring as the state of each simulated patient is observed and actions updated every five days. The reward is a function of T-cell counts, free HIV viruses, anti-HIV immune response, and side effects. We collect 10,000 offline samples randomly before training, then run both CQL and our algorithm for $N=1,000$ online episodes respectively. We plot the discounted reward and the number of exploration constraint violations corresponding to various exploration budgets $\eta$ lying between $40$ and $100$. Table~\ref{table:cql_comparison} shows the comparison of maximum discounted returns of our algorithm versus CQL, and $\epsilon$-greedy Q-learning with an annealing $\epsilon$ schedule. We have additionally run soft actor-critic (SAC) for comparison, but the SAC agent performs poorly in this environment and quickly gets stuck in a constant suboptimal policy. Compared to $\epsilon$-greedy, our algorithm significantly reduces the number of constraint violations without sacrificing performance.

We would like to note that our proposal to extend our algorithm to MDPs with continuous states are purely empirical, as the extension does not satisfy our assumptions on tabular MDPs and will thus make theoretical analysis significantly harder, taking it beyond the scope of this work. For future work, we are thinking about using linear function approximation to model MDPs with continuous states and defining a “small distance” between states as similarity between linear features.

\textbf{Inventory control.}
We consider a single-product stochastic inventory control problem based on \cite{garcelon2020conservative}, but with a finite horizon. At the beginning of each month $t$, the manager notes the current inventory of a single product, and then decide the number of items to order from a supplier before observing the random demand. They have to account for the tradeoff between the costs of keeping inventory and lost sales or penalties resulting from being unable to satisfy customer demand. The objective is to maximize profit during the entire decision-making process.

The state space is the number of items in the inventory, $S=\{0, \dots, M\}$, where $M=5$
is the maximum capacity. The action space is $A_s=\{0, \dots, M-s\}$ for each state $s\in S$. Given inventory state $s_t$ at the beginning of month $t$, the number of items $a_t$ to order is determined by the manager. We assume that a time-homogeneous uniform distribution $D_t$ generates the random demand of each month $t$, and that the horizon is $H=20$. The inventory at the beginning of month $t+1$ is given by
\begin{equation*}
s_{t+1}=\max\{0, s_t+a_t-D_t\}.
\end{equation*}
Next, we define the associated cost functions. We assume a fixed cost $K=2$ for placing orders and a variable cost $c(u)=2u$ that increases with the quantity ordered:
\begin{align*}
O(u)=
\begin{cases}
K+c(u)~&\text{if }u>0\\
0&\text{if }u=0.
\end{cases}
\end{align*}
The cost of maintaining an inventory of $u$ units for a month is represented by the nondecreasing function $h(u)=u$. If the demand is $j$ units and sufficient inventory is available to meet the demand, the manager receives a revenue of $f(j)$. Finally, the reward is defined as $r(s_t, a_t, s_{t+1})=-O(a_t)-h(s_t+a_t)+f(s_t+a_t-s_{t+1})$, where we take $f(u)=8u$ in our experiments. We normalize the rewards so that they are supported in $[0,1]$.  We use an offline dataset of $1,500$ randomly generated past episodes to warm-start the algorithms, and then compute the regret and number of exploration constraint violations corresponding to various exploration budgets $\eta$. We use $N=50,000$ total episodes.

Figure~\ref{fig:exp_invcon} (a) shows that the regret of our algorithm starts out linearly increasing, since in the beginning the meta algorithm is forced to switch to the baseline policy a lot in one meta-episode to satisfy the constraint. As historical data accrues, our algorithm uses the UCBVI policy more frequently since its reward deficit decreases. At some point, our algorithm starts to converge at a similar rate as UCBVI. Note that UCBVI converges faster since it ignores the exploration constraint and can explore arbitrarily, even if some actions result in poor values. Figure~\ref{fig:exp_invcon} (b) shows the number of times the exploration constraint is violated. Our algorithm almost always satisfies the constraint for all shown values of $\eta$, whereas UCBVI fails to do so in a significant number of episodes, especially when $\eta$ is small.

\begin{figure*}[t]
\centering
\includegraphics[width=0.65\linewidth]{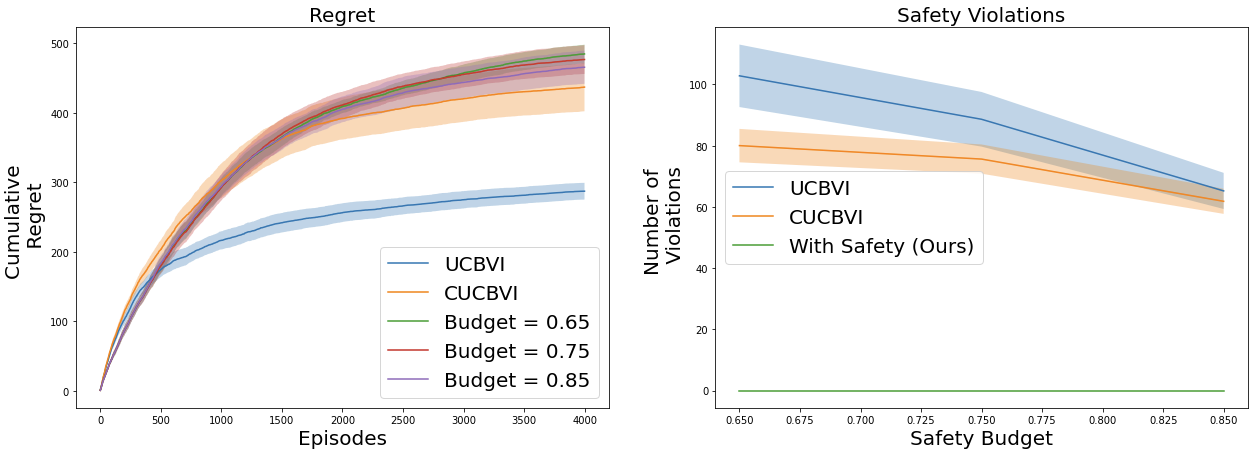}
\captionof{figure}{Inventory control: regret (left) and exploration constraint violations (right) of our approach with different exploration budgets $\eta$ vs. UCBVI.}
\label{fig:exp_invcon}
\end{figure*}

\section{Further Discussions}\label{sec:disappendix}
\textbf{Contributions.} Our main contributions are: (i) We propose a new notion of “uniformly conservative exploration” in equation \eqref{eqn:safety} for MDPs, (ii) we devise a novel “meta-episodic” online reinforcement learning algorithm to maintain this exploration constraint, and (iii) we prove that our algorithm ensures uniformly conservative exploration while achieving sublinear regret. Importantly, our “meta-episodic” strategy is both a novel algorithmic approach for ensuring unbiased exploration, and also requires novel proof techniques to ensure bounded regret.

\textbf{Related work.} There has been a great deal of recent interest in safe reinforcement learning \citep{garcia2015comprehensive}, although it has largely focused on guaranteeing safety rather than proving regret bounds (that guarantee convergence to an optimal policy). Furthermore, most of these approaches focus on safety constraints in the form of safe regions, where the goal is to stay inside the safe region \citep{li2020robust}. Such constraints are common in robotics, but less so for other applications of reinforcement learning such as healthcare, education, and operations research. In contrast, our approach studies a conservative exploration approach that focuses on avoiding underperforming an existing policy, which is more applicable in these settings.

\end{document}